\documentclass[runningheads]{llncs}

\usepackage{amsmath,amssymb,mathtools}
\usepackage{bm}
\usepackage{microtype}
\usepackage{graphicx}
\usepackage{booktabs}
\usepackage{multirow}
\usepackage{siunitx}
\usepackage[hidelinks]{hyperref}
\usepackage{xspace}

\newcommand{\R}{\mathbb{R}}
\newcommand{\Id}{\mathrm{Id}}

\newcommand{\CT}{C_T}
\newcommand{\Lam}{\Lambda}
\newcommand{\C}{\mathcal{C}}
\newcommand{\Vect}{\mathbf{Vect}}
\newcommand{\GAP}{\mathrm{GAP}}
\newcommand{\convstar}{\ast}

\title{Learning with Category-Equivariant Architectures for Human Activity Recognition}
\titlerunning{Learning with Category-Equivariant Architectures}

\author{Yoshihiro Maruyama\inst{1,2}}
\authorrunning{Y. Maruyama}
\institute{School of Informatics, Nagoya University, Japan\\
\email{maruyama@i.nagoya-u.ac.jp}
\and
School of Computing, Australian National University, Australia\\
\email{yoshihiro.maruyama@anu.edu.au}
}

\begin{document}
\maketitle

\begin{abstract}
We propose CatEquiv, a category‑equivariant neural network for Human Activity Recognition (HAR) from inertial sensors that systematically encodes temporal, amplitude, and structural symmetries. We introduce a symmetry category that jointly represents cyclic time shifts, positive gain scalings, and the sensor‑hierarchy poset, capturing the categorical symmetry structure of the data. CatEquiv is equivariant with respect to this symmetry category. On UCI‑HAR under out‑of‑distribution perturbations, CatEquiv attains markedly higher robustness compared with circularly padded CNNs  and plain CNNs. These results demonstrate that enforcing categorical symmetries yields strong invariance and generalization without additional model capacity.
\end{abstract}

\section{Introduction}\label{sec:intro}

\paragraph{Motivation.}
Human Activity Recognition (HAR) from smartphone inertial sensors must contend with variability that is \emph{structural}, not merely random noise: windows begin at different phases (temporal shifts), phones are held or worn at arbitrary orientations (3D rotations), sensor gains drift over time (amplitude scaling), and channels are related hierarchically (axes $\rightarrow$ sensor $\rightarrow$ fused signals). Standard CNN/MLP baselines learn coordinate‑specific templates; they perform well in‑distribution but degrade sharply once any of these factors change at test time, a pattern broadly observed in robustness studies on distribution shift~\cite{HendrycksDietterich2019Corruptions}. 

\paragraph{Our approach.}
We target this failure mode on UCI‑HAR~\cite{Anguita2013UCIHAR} using the raw six‑channel inertial streams (accelerometer and gyroscope, three axes each). To emulate realistic deployment, we evaluate under composite out‑of‑distribution (OOD) conditions that simultaneously apply cyclic time shifts (to model phase mismatch), independent random $\mathrm{SO}(3)$ rotations per tri‑axial block (to model device pose), and per‑sensor gain changes (to model calibration drift). These perturbations are not adversarial; they are the natural algebra of how signals vary in practice. A learning principle that builds these symmetries into the representation, rather than fighting them with ad‑hoc augmentation, should therefore confer robustness by construction.

\paragraph{Equivariant learning.}
The geometric deep learning view emphasizes equivariance to symmetry groups as a principled route to generalization~\cite{Bronstein2021GDL}.
Group-equivariant CNNs~\cite{CohenWelling2016GCNN,CohenWelling2017Steerable}, steerable / E(2)-equivariant networks~\cite{WeilerCesa2019E2,CohenWelling2017Steerable}, spherical CNNs~\cite{Cohen2018Spherical,Esteves2018SO3}, and 3D $\mathrm{SE}(3)$-equivariant architectures~\cite{Thomas2018TFN,Fuchs2020SE3T,Satorras2021EGNN,Anderson2019Cormorant} instantiate this idea across domains (see also~\cite{Cohen2019Gauge,Worrall2017Harmonic}).
Parameter sharing yields equivariance by construction~\cite{Ravanbakhsh2017ParamSharing}, and convolution can be generalized to compact-group actions~\cite{KondorTrivedi2018Compact}.
Beyond groups, \emph{set} and \emph{graph} symmetries have been captured via Deep Sets~\cite{Zaheer2017DeepSets} and invariant/equivariant graph networks~\cite{Maron2019IEGN}. Our work extends these ideas for HAR by combining group actions (time, gain) with a \emph{poset} describing the sensor hierarchy within a product category, framing the linear core as a natural transformation between functors whose naturality squares commute with all morphisms in the product category.

\paragraph{From groups to categories.}
Many symmetries in real-world data are arguably combinations of various types of symmetries rather than purely group-theoretic symmetries. In the UCI-HAR dataset, for example, time-window shifts form a cyclic group; gains form a multiplicative group; the sensor stack (axes $\to$ sensor $\to$ TOTAL) is naturally a thin category (poset). Category theory~\cite{FongSpivak2019SevenSketches} provides a fundamental mathematical language to unify such structures and reason about \emph{naturality} of learned maps. Concretely, we work with the product category
\[
\C_3 \;=\; \mathbf{B}(\CT \times \Lam) \times P,
\]
where \(\mathbf{B}(\CT\times\Lam)\) is the one-object category induced by cyclic time shifts \(\CT\) and positive gains \(\Lam\) and \(P\) is the sensor-hierarchy poset. Given functors \(X,Y:\C_3\to\Vect\), a family \(\eta\) is a \emph{natural transformation} if for every morphism \(g\) in \(\C_3\) the naturality law holds: \(Y(g)\,\eta = \eta\,X(g)\).\footnote{$\Vect$ denotes the category of finite dimensional real vector spaces.}

\paragraph{Our method: CatEquiv.}
We introduce \emph{CatEquiv}, a category-aware neural network architecture whose linear core realizes a natural transformation \(\eta:X\Rightarrow Y\) between functors \(X,Y:\C_3\to\Vect\) via architectural constraints:
(i) circular 1D convolutions and global time pooling (time-shift equivariance/invariance);
(ii) per-sensor RMS normalization plus log-RMS side channels (gain invariance with controlled amplitude cues);
(iii) axis-shared temporal filters followed by $\ell_2$ pooling across axes (rotation invariance at readout);
(iv) sensor-shared filters and averaging (poset consistency). 
This realizes \emph{category-equivariant representation learning}: the linear core commutes with the morphisms of \(\C_3\) by construction, and the readout implements the corresponding invariants.

\paragraph{Baselines.}
To isolate the contribution of each symmetry, we compare CatEquiv against:
(a) \emph{PlainCNN}—a two-layer 1D CNN with zero padding (no explicit symmetry handling beyond translational weight sharing);
(b) \emph{CircCNN}—the same network with circular padding (time-shift equivariance only).
This mirrors the progression from no categorical structure, to \(\mathbf{B}(\CT)\) only, to the full \(\mathbf{B}(\CT\times\Lam)\times P\).

\paragraph{Results in brief.}
Under the composite OOD ($\pm$18 cyclic time shift, random \(\mathrm{SO}(3)\) rotations, sensor-wise gain in $[0.7,1.4]$), CatEquiv achieves substantially higher accuracy and macro-F1 than both baselines (e.g., 0.73 F1 vs.\ 0.42 for CircCNN and 0.12 for PlainCNN).
The gain stems from enforcing the categorical symmetry rather than increasing model capacity in a brute-force manner.

\paragraph{Contributions.}
\begin{itemize}
  \item \textit{A category-equivariant HAR model.} We formalize and implement \(\C_3=\mathbf{B}(\CT\times\Lam)\times P\) for inertial sensing, yielding a functorial deep architecture whose linear core is natural and whose readout is invariant: time-shift equivariant/invariant, gain-robust, rotation-invariant at readout, and poset-consistent.
  \item \textit{Fair baselines and ablations.} We disentangle the effect of time-shift equivariance (CircCNN) from the full category (CatEquiv), and quantify the contribution of each component (axis sharing, $\ell_2$ pooling, RMS/log‑RMS, sensor tying, dilations).
  \item \textit{Robust OOD performance on UCI-HAR.} On raw streams with matched train-time augmentation, CatEquiv delivers large gains over CNN baselines under joint time/rotation/gain shifts.
\end{itemize}

\paragraph{Relation to prior work.}
CatEquiv connects the group-equivariant CNN literature~\cite{CohenWelling2016GCNN,WeilerCesa2019E2,Cohen2018Spherical,Esteves2018SO3}, 3D \(\mathrm{SE}(3)\)-equivariant models~\cite{Thomas2018TFN,Fuchs2020SE3T,Satorras2021EGNN,Anderson2019Cormorant}, invariant scattering~\cite{Mallat2012Scattering,SifreMallat2013RotScale}, and parameter-sharing views of equivariance~\cite{Ravanbakhsh2017ParamSharing,KondorTrivedi2018Compact}, while extending beyond pure groups to more general categorical symmetry structures.
For sensor fusion, our sensor-averaging readout echoes permutation-invariant designs~\cite{Zaheer2017DeepSets} but is constrained by the sensor poset rather than a flat set.

\paragraph{Outline.}
Section~\ref{sec:catequiv} details the category and the CatEquiv architecture. Section~\ref{sec:exp} presents the dataset, OOD protocol, models, and results, including ablations. We conclude with a brief summary and remarks. The appendix provides mathematical foundational results.

\section{CatEquiv: Category-Equivariant Neural Networks}\label{sec:catequiv}

We introduce the minimal foundations of CatEquiv that are required for the experiments below. 

\subsection{Symmetry category and functorial modeling}\label{subsec:catequiv-category}
We model HAR symmetries by the product category
\begin{equation}\label{eq:C3-def}
  \C_3 \;=\; \underbrace{\mathbf{B}(\CT \times \Lam)}_{\text{time shift/per-sensor gain}}
  \;\times\; \underbrace{P}_{\text{sensor hierarchy}},
\end{equation}
where:
\begin{itemize}
  \item \(\CT = \mathbb{Z}/T\mathbb{Z}\) (cyclic time shifts for a \(T\)-length window; a finite cyclic group used via its one‑object category \(\mathbf{B}(\CT)\)), and
  \[
     \Lam \;=\; \R_{>0}^{\{\mathrm{ACC},\mathrm{GYR}\}}
  \]
  (per‑sensor positive gains, a commutative group under component‑wise multiplication $(\lambda'\cdot\lambda)_s=\lambda^{'}_{s}\lambda_{s}$). The one‑object category \(\mathbf{B}(\CT\times\Lam)\) is induced by the direct‑product group \(\CT\times\Lam\): it has a single object \(\star\) and morphisms \(m=(\tau,\lambda)\in \CT\times\Lam\), with composition 
  \[
  (\tau_2,\lambda_2)\circ(\tau_1,\lambda_1)\;=\;(\tau_2{+}\tau_1,\;\lambda_2\cdot\lambda_1).
  \]
  The actions of \(\CT\) (cyclic shift) and \(\Lam\) (sensor‑wise scaling) on signals commute.
\item \(P\) is the poset (thin category) whose underlying set is
\[
\{\mathrm{ACC}_{x},\mathrm{ACC}_{y},\mathrm{ACC}_{z},\mathrm{GYR}_{x},\mathrm{GYR}_{y},\mathrm{GYR}_{z},\mathrm{ACC},\mathrm{GYR},\mathrm{TOTAL}\}.
\]
Its (partial) ordering is generated by
\[
\mathrm{ACC}_\alpha \prec \mathrm{ACC} \prec \mathrm{TOTAL},
\qquad
\mathrm{GYR}_\alpha \prec \mathrm{GYR} \prec \mathrm{TOTAL}
\quad(\alpha\in\{x,y,z\}).
\]
\end{itemize}

\paragraph{Data functor.}
Let \(X:\C_3\to\Vect\) be the \emph{data functor} that assigns to each object 
\((\star,s)\in {\rm Obj}(\C_3)\) a real vector space 
\(X(\star,s)\) of time-series signals (channels \(\times\) time);
concretely, \(X(\star,s)\cong\R^{C_s\times T}\), where \(C_s\) is the number of channels associated with \(s\).
Denote by \(x\in X(\star,\mathrm{ACC}_{\alpha})\cong\R^{T}\) a single‑axis stream and by
\(x\in X(\star,\mathrm{ACC})\cong\R^{3\times T}\) a tri‑axial sensor stream.

\medskip\noindent\emph{Per–sensor gain as a block scaling.}
For \(\lambda=(\lambda_{\mathrm{ACC}},\lambda_{\mathrm{GYR}})\in\Lambda=\R_{>0}^{\{\mathrm{ACC},\mathrm{GYR}\}}\), define
\[
(\lambda\odot x)_{s,\alpha}(t)=\lambda_s\,x_{s,\alpha}(t)
\quad\text{for } s\in\{\mathrm{ACC},\mathrm{GYR}\},\ \alpha\in\{x,y,z\},\ t=1,\ldots,T,
\]
so that for sensor blocks \(x_s\in\R^{3\times T}\),
\((\lambda\odot x)_s=\lambda_s\,x_s\), and for the concatenated block
\(x_{\mathrm{TOTAL}}=(x_{\mathrm{ACC}},x_{\mathrm{GYR}})\in\R^{6\times T}\),
\[
(\lambda\odot x)_{\mathrm{TOTAL}}
=\big(\lambda_{\mathrm{ACC}}\,x_{\mathrm{ACC}},\ \lambda_{\mathrm{GYR}}\,x_{\mathrm{GYR}}\big).
\]

\medskip\noindent\emph{Unified time–gain action via a representation.}
For each \[s\in\{\mathrm{ACC}x,\mathrm{ACC}y,\mathrm{ACC}z,\mathrm{GYR}x,\mathrm{GYR}y,\mathrm{GYR}z,\mathrm{ACC},\mathrm{GYR},\mathrm{TOTAL}\}\]
define a representation \(\rho_s:\Lambda\to GL\!\big(X(\star,s)\big)\) by
\[
\rho_s(\lambda):=
\begin{cases}
\lambda_{\mathrm{ACC}}\,\big(I_{C_s}\otimes I_T\big) & \text{if } s\preceq \mathrm{ACC},\\[2pt]
\lambda_{\mathrm{GYR}}\,\big(I_{C_s}\otimes I_T\big) & \text{if } s\preceq \mathrm{GYR},\\[2pt]
\big(\operatorname{diag}(\lambda_{\mathrm{ACC}} I_{3},\,\lambda_{\mathrm{GYR}} I_{3})\otimes I_T\big) & \text{if } s=\mathrm{TOTAL},
\end{cases}
\]
i.e., \(\rho_s(\lambda)\) multiplies all channels belonging to sensor \(s\) by the appropriate gain,
extended trivially over time. Let \(\tau_\Delta:X(\star,s)\to X(\star,s)\) be the cyclic time–shift
\((\tau_\Delta x)_{c,t}=x_{c,\,t-\tau\!\!\!\pmod T}\). Define
\begin{equation}\label{eq:data-monoid}
  S^{(s)}_{(\tau,\lambda)}
  \;:=\;
  \rho_s(\lambda)\circ \tau_\Delta.
\end{equation}
Because \(\rho_s(\lambda)\) acts on channels and \(\tau_\Delta\) on time, they commute:
\(S^{(s)}_{(\tau_2,\lambda_2)}S^{(s)}_{(\tau_1,\lambda_1)}=S^{(s)}_{(\tau_2+\tau_1,\;\lambda_2\lambda_1)}\).
For axis objects we inherit the sensor gain, e.g.
\(\rho_{\mathrm{ACC}\alpha}(\lambda)=\lambda_{\mathrm{ACC}}\,I_{\R^T}\) and
\(\rho_{\mathrm{GYR}\alpha}(\lambda)=\lambda_{\mathrm{GYR}}\,I_{\R^T}\).

\medskip
Define the canonical injections along the poset by
\begin{equation}\label{eq:axis-incl}
  J_{\text{axis}_{s,\alpha}\to \text{sensor }s}
  \;=\; j_{s,\alpha}:\ \R^{T} \to \R^{3\times T},
\end{equation}
\begin{equation}\label{eq:sensor-inj}
  J_{\text{sensor }s\to \text{TOTAL}}
  \;=\; i_{s}:\ \R^{3\times T} \to \R^{6\times T}, 
\end{equation}
with
\(j_{s,\alpha}(v)=(0,\ldots,v,\ldots,0)\),
\(i_{\mathrm{ACC}}(x)=(x,0)\),
\(i_{\mathrm{GYR}}(x)=(0,x)\).
For a morphism \(\big((\tau,\lambda),u:s\!\to\!t\big)\) in \(\mathbf{B}(\CT\times\Lambda)\times P\), set
\begin{equation}\label{eq:X-general}
  X\big((\tau,\lambda),u\big)
  \;:=\; J_u \circ S^{(s)}_{(\tau,\lambda)}
  \;=\; S^{(t)}_{(\tau,\lambda)} \circ J_u
  \;:\; X(\star,s)\to X(\star,t),
\end{equation}
where \(S^{(s)}_{(\tau,\lambda)}\) and \(S^{(t)}_{(\tau,\lambda)}\) are as in \eqref{eq:data-monoid}.
By construction \(\rho_t\) extends \(\rho_s\) along \(u:s\!\to\!t\),
so \(J_u\,\rho_s(\lambda)=\rho_t(\lambda)\,J_u\), and the equality in \eqref{eq:X-general}
follows (time shifts commute with \(J_u\) as well).

\paragraph{Linear core as a natural transformation.}
Let us define the \emph{feature functor} \(Y:\C_3\to\Vect\) analogously to the data functor \(X\).
On objects, \(Y(\star,s)\) is the feature space with the same sensor–block decomposition as \(X(\star,s)\).
On morphisms, for an arrow \(((\tau,\lambda),u:s\!\to\!t)\) with
\((\tau,\lambda)\in \CT\times\Lambda\) and \(u\in P\), define
\[
Y\big((\tau,\lambda),u\big)
\;:=\; J_u \circ S^{(s)}_{(\tau,\lambda)}
\;=\; S^{(t)}_{(\tau,\lambda)} \circ J_u,
\]
where \(S^{(r)}_{(\tau,\lambda)}:Y(\star,r)\!\to\!Y(\star,r)\) is the time–gain action
on \(Y\)-spaces (given by the same formula as in \eqref{eq:data-monoid}, with
\(X\) replaced by \(Y\)), and
\(J_u\in\{\Id,\,j_{s,\alpha},\,i_s,\,i_s\!\circ j_{s,\alpha}\}\) is the canonical
inclusion induced by \(u\) (the same arrow–shapes as in
\eqref{eq:axis-incl}–\eqref{eq:sensor-inj}, acting on \(Y\)-spaces).
The \emph{linear core} is the family of linear maps
\[
  \eta_{(\star,s)}:\;X(\star,s)\to Y(\star,s),
\]
obtained by keeping only linear operators (circular convolutions, the canonical
injections \eqref{eq:axis-incl}–\eqref{eq:sensor-inj}, depthwise circular box
smoothing, and concatenation/direct sums).
Equivariance (naturality) is defined as follows:  
For every morphism \(((\tau,\lambda),u:s\!\to\!t)\),
\begin{equation}\label{eq:naturality}
  Y\big((\tau,\lambda),u\big)\,\eta_{(\star,s)}
  \;=\; \eta_{(\star,t)}\,X\big((\tau,\lambda),u\big).
\end{equation}
The nonlinear reductions used for readout are handled separately.
Naturality (equivariance) of the linear core with respect to the morphisms
of \(\C_3\) is ensured by using (i) circular temporal convolutions (commuting with
\(C_T\)), and (ii) block–diagonal linear maps in the axis and sensor decompositions (depthwise in Stage--1 and grouped in Stage--2 defined below) so that the (channel–lifted)
canonical injections along \(P\) commute (i.e., no cross–axis/sensor mixing).

\subsection{Specification of CatEquiv}\label{subsec:catequiv-formal}

For a sensor \(s\in\{\text{ACC},\text{GYR}\}\), define the per‑window energy and scales
\begin{equation*}
\begin{aligned}
R_s(x) &:= \frac{1}{3T}\sum_{a\in\{x,y,z\}}\sum_{t=1}^T x_{s,a}(t)^2,\\
\rho^{\mathrm{norm}}_s(x) &:= \max\!\bigl(\varepsilon,\sqrt{R_s(x)}\bigr),\\
\mathcal{N}_s(x) &:= \frac{x}{\rho^{\mathrm{norm}}_s(x)},\\
r_s(x) &:= \tfrac{1}{2}\log R_s(x).
\end{aligned}
\end{equation*}

Let \(x\in\R^{T\times 2\times 3}\) be a window (time \(\times\) sensors \(\times\) axes). 
Define the gain‑processed input:
\[
\begin{aligned}
& \widehat x_s = x_s / \rho^{\mathrm{norm}}_s \in \R^{3\times T},\qquad
  r_s = \tfrac{1}{2}\log R_s \in \R,\\
& X_{\text{axes}} = \mathrm{stack}(\widehat x_{\text{ACC}},\widehat x_{\text{GYR}})\in\R^{6\times T},\\
& X_{\log} = \operatorname{Rep}_T(r_{\text{ACC}},r_{\text{GYR}})\in\R^{2\times T},\qquad
  X_{\log}(t)\equiv(r_{\text{ACC}},r_{\text{GYR}}).
\end{aligned}
\]

Unless stated otherwise we represent signals as \emph{channels} $\times$ \emph{time}.
Thus, after reshaping the raw window $x\in\mathbb{R}^{T\times 2\times 3}$ we work with
$X_{\text{axes}}\in\mathbb{R}^{6\times T}$ (six axis channels stacked over time), and
all convolutions and smoothers act along the time dimension (length $T$).

A superscript/glyph ``$\circlearrowleft$'' attached to a 1‑D time operator means
\emph{circular (wrap‑around) padding} along time, i.e. indices are taken modulo $T$.
For example, $\mathrm{Conv}^{\circlearrowleft}$ is 1‑D convolution with circular padding and
$\mathrm{Box}^{\circlearrowleft}_k$ is a depthwise circular $k$‑tap averaging filter.

$\operatorname{Rep}_T:\mathbb{R}^2\to\mathbb{R}^{2\times T}$ replicates a vector
across time: for $u\in\mathbb{R}^2$, $(\operatorname{Rep}_T(u))(t)=u$ for
$t=1,\dots,T$. Hence $X_{\log}=\operatorname{Rep}_T(r_{\mathrm{ACC}},r_{\mathrm{GYR}})\in\mathbb{R}^{2\times T}$.

CatEquiv computes:

\begin{subequations}\label{eq:forward}

\noindent\textbf{Stage 1 (axes, linear).}
\begin{equation}\label{eq:s1}
  H_1 \;=\; \mathrm{Conv}^{\circlearrowleft}_{\text{axis}}\!\big(X_{\text{axes}}; W_{\text{ax}},\kappa_1\big)
  \;\in\; \R^{(6C_1)\times T}.
\end{equation}

\noindent\textbf{Axis}\textrightarrow\textbf{Sensor (invariant reduction).}
\begin{equation}\label{eq:axis-sensor}
  S \;=\; \big(\, \|H_1^{\text{ACC}}\|_2,\, \|H_1^{\text{GYR}}\|_2 \,\big)
  \;\in\; \R^{(2C_1)\times T}.
\end{equation}

\noindent\textbf{Per-seq GroupNorm.}
\begin{equation}\label{eq:gn}
  \widetilde S \;=\; \mathrm{GN}_{\text{groups}=2} \big(S\big).
\end{equation}

\noindent\textbf{Stage 2 (sensor, multi-scale).}
\begin{equation}\label{eq:s2}
  H_{2}^{(d)} \;=\; \phi \Big(\mathrm{Conv}^{\circlearrowleft}_{\text{sens},\,d} \big(\widetilde S; W_d,\kappa_2\big)\Big),
  \qquad d\in\{1,2,3\}.
\end{equation}

\noindent\textbf{Sensor fusion (TOTAL, readout).}
\begin{equation}\label{eq:sensor-total}
  T^{(d)} \;=\; \mathrm{mean}_{\text{sensor}}\!\big(H_2^{(d)}\big)
  \;\in\; \R^{C_2^{(d)}\times T}.
\end{equation}

\noindent\textbf{Smoothing + GAP.}\footnote{The depthwise temporal box filter $\mathrm{Box}^{\circlearrowleft}_k$ acts only on time and uses the same kernel for all sensors/channels; hence it commutes with the sensor mean in \eqref{eq:sensor-total} and with $\GAP_t$. Equivalently, one may apply $\mathrm{Box}^{\circlearrowleft}_k$ to $H^{(d)}_2$ before \eqref{eq:sensor-total} without changing $g^{(d)}$ after $\GAP_t$.}
\begin{equation}\label{eq:smooth-gap}
  g^{(d)} \;=\; \GAP_t \Big(\mathrm{Box}^{\circlearrowleft}_{k} \big(T^{(d)}\big)\Big)
  \;\in\; \R^{C_2^{(d)}}.
\end{equation}

\noindent\textbf{Head fusion.}
\begin{equation}\label{eq:head}
  z \;=\; \big[\, g^{(1)} \,\|\, g^{(2)} \,\|\, g^{(3)} \,\|\, \GAP_t(X_{\log}) \,\big] \;\in\; \R^{D},
  \qquad \text{logits} \;=\; W_{\text{head}} z + b.
\end{equation}

\end{subequations}

Here $\mathrm{Conv}^{\circlearrowleft}_{\text{axis}}$ denotes depthwise 1‑D convolution with circular padding, with the same kernel bank applied to each axis channel (explicit parameter tying); $\mathrm{Conv}^{\circlearrowleft}_{\text{sens},\,d}$ denotes grouped 1‑D convolution with circular padding and dilation $d$, with the same kernel bank for each sensor. A scalar nonlinearity $\psi=\mathrm{ReLU}$ is applied after the axis $\ell_2$ reduction in \eqref{eq:axis-sensor} to preserve $O(3)$ invariance, and $\phi=\mathrm{ReLU}$ is used in Stage‑2. The concatenation $[\cdot\| \cdot]$ stacks feature vectors. We carry $X_{\log}\in\mathbb{R}^{2\times T}$ as two input channels for bookkeeping, but Stage‑1 and Stage‑2 operate only on the first six (axis) channels; the $X_{\log}$ channels bypass the convolutional stacks and are fused at the head via global time averaging, $\GAP_t(X_{\log})$.

If Stage‑1 has \(C_1\) channels per axis, \(H_1\in\R^{(6C_1)\times T}\). After \(\ell_2\) aggregation \eqref{eq:axis-sensor}, \(S\in\R^{(2C_1)\times T}\). Each sensor‑shared branch with output \(C_2^{(d)}\) channels yields \(g^{(d)}\in\R^{C_2^{(d)}}\). With three branches and two \(\log\)RMS scalars, the head input has \(D=\sum_d C_2^{(d)}+2\) channels.

\subsection{Remarks}\label{subsec:impl}

\paragraph{Depthwise and grouped convolutions.}
Stage‑1 uses depthwise 1‑D conv with groups \(=6\) and explicit parameter tying so that the same kernel bank is applied to each of the six axis channels (axis‑shared), parameter cost \(C_1\kappa_1\); the pointwise nonlinearity is applied after the axis‑norm to preserve \(\mathrm{O}(3)\) invariance. Stage‑2 uses grouped conv with groups \(=2\) and explicit tying across sensors (sensor‑shared), cost \(\sum_{d} C_2^{(d)} C_1 \kappa_2\). Box smoothing is depthwise with groups \(=\sum_d C_2^{(d)}\). All convolutions use circular padding, preserving temporal length \(T\). Because the axis $\ell_2$ reduction removes orientation and reflections alike, the readout is $O(3)$-invariant even though the physical perturbations during testing
are rotations in $SO(3)$.

\paragraph{Normalization.}
GroupNorm with \(\text{groups}=2\) across channel groups \((\text{ACC},\text{GYR})\) acts as per‑sequence, per‑sensor instance normalization and commutes with \(\CT\).

\paragraph{Summary.}
CatEquiv consists of a natural (i.e., equivariant) linear core \(\eta:X\Rightarrow Y\) with \(X,Y:\C_3\to\Vect\), where \(\C_3=\mathbf{B}(\CT\times\Lambda)\times P\), followed by an invariant readout. Each linear layer commutes with the morphisms of \(\C_3\) by construction (circular convolutions for time; canonical injections for the poset), yielding the desired equivariances, while RMS/\(\log\)RMS and axis-norm-plus-nonlinearity produce the readout invariances; sensor fusion preserves them, and multi-dilation branches provide multi-scale context while preserving equivariance.

\section{Experiments and Results}\label{sec:exp}

\subsection{Dataset and Preprocessing}\label{subsec:data}

\paragraph{UCI-HAR (inertial streams).}
We use the public UCI-HAR dataset~\cite{Anguita2013UCIHAR} with the official \texttt{train}/\texttt{test} split.
Each example is a fixed-length window of \(T{=}128\) time steps comprising two tri-axial sensors:
accelerometer (ACC) and gyroscope (GYR), hence \(6\) raw channels (\(2{\times}3\)).
We use the raw inertial streams.

\paragraph{Per-sensor gain processing.}
For each window and sensor \(s\in\{\mathrm{ACC},\mathrm{GYR}\}\) define
\[
R_s(x) \;=\; \frac{1}{3T}\sum_{a\in\{x,y,z\}}\sum_{t=1}^T x_{s,a}(t)^2,\qquad
\rho^{\mathrm{norm}}_s(x) \;=\; \max\!\big(\varepsilon,\,\sqrt{R_s(x)}\big).
\]
We form gain-invariant streams \(\widehat x_s = x_s/\rho^{\mathrm{norm}}_s(x)\) and append two log-RMS side channels
\(r_s = \tfrac12\log R_s(x)\) replicated along time.
The final input tensor has \(8\) channels: \(6\) normalized axes \(+\) \(2\) log-RMS channels.

\subsection{OOD Protocol}\label{subsec:ood}
We evaluate robustness under a composite OOD transformation applied per window:
\begin{enumerate}\itemsep0.2em
  \item \textbf{Time shift} \(\Delta \sim \mathrm{Unif}\{-18,\dots,18\}\), applied cyclically (wrap-around), the same \(\Delta\) to all channels.
  \item \textbf{Gain drift} per sensor \(g_s \sim \mathrm{Unif}[0.7,\,1.4]\); raw streams are scaled \(x_s\mapsto g_s x_s\).
  \item \textbf{Random rotation} Random rotation \(R \sim \mathrm{SO}(3)\) once per window (Haar via QR with sign correction~\cite{Mezzadri2007}), applied to both ACC and GYR: \(x_s \mapsto R\,x_s\).
\end{enumerate}

\subsection{Models}\label{subsec:models}

We compare three architectures with approximately comparable capacity.

\paragraph{PlainCNN (zero padding).}
Two 1-D convolutions with kernel sizes \(k_1{=}9\), \(k_2{=}9\), zero padding, ReLU, dropout,
global average pooling (GAP) over time, linear classifier.
This baseline lacks explicit symmetry handling beyond translational weight sharing.

\paragraph{CircCNN (circular padding).}
Same as PlainCNN but using circular padding in all convolutions,
making the stack \emph{time-shift equivariant} (invariance after GAP).

\paragraph{CatEquiv.}
The proposed category-equivariant model (\S\ref{sec:catequiv}):
\begin{itemize}\itemsep0.2em
  \item \textbf{Stage-1 (axes).} Depthwise (axis-shared) circular 1-D convolution with \(C_1\) channels per axis (\(k_1{=}9\)).\footnote{No dropout is applied in Stage-1, and no pointwise nonlinearity is applied before axis aggregation to preserve \(\mathrm{O}(3)\) invariance of the reduction.}
  \item \textbf{Axis\(\to\)Sensor.} \(\ell_2\)-magnitude across the \(x,y,z\) axes (per sensor), then ReLU; this yields \(\mathrm{O}(3)\)-invariant per-sensor features while keeping the nonlinearity invariant.
  \item \textbf{Per-sequence GroupNorm.} GroupNorm with groups\(=2\) (one per sensor) on the sensor-stacked channels.
  \item \textbf{Stage-2 (sensor, multi-scale).} Three sensor-shared (weights tied across ACC and GYR) circular conv branches with dilations \(d\in\{1,2,3\}\) and kernel sizes \(k_2\in\{9,11,15\}\); ReLU after each branch. 
  \item \textbf{Sensor fusion.} Average over the sensor dimension (ACC, GYR) to form TOTAL.
  \item \textbf{Temporal smoothing + GAP.} Depthwise circular box filter (e.g., \(k{=}5\)) followed by global average pooling over time; the filter acts on time only and uses the same kernel across sensors/channels, so it commutes with the sensor mean and with \(\GAP_t\).
  \item \textbf{Head fusion.} Concatenate the three multi-scale descriptors with the time-averaged log-RMS channels, i.e., \([\,g^{(1)} \,\|\, g^{(2)} \,\|\, g^{(3)} \,\|\, \GAP_t(X_{\log})\,]\); apply dropout on this head descriptor before the linear classifier.
\end{itemize}
Unless stated otherwise, we use \(C_1{=}32\) and \(C_2{=}\{64,32,32\}\) for the three branches. 
All convolutions use circular padding with odd kernels, preserving temporal length \(T\) and exact shift equivariance.

\subsection{Training Setup}\label{subsec:train}

We train all models with the Adam optimizer (learning rate \(10^{-3}\),
weight decay \(5{\times}10^{-4}\), $\beta_1{=}0.9$, $\beta_2{=}0.999$)
and batch size \(128\).
The Adam $\varepsilon$ parameter is fixed at $\varepsilon=10^{-8}$ for numerical stability.
Gradient norms are clipped to $\|\nabla\theta\|_2 \le 5.0$ at every update.
We use ReduceLROnPlateau (factor \(0.5\), patience \(3\)) and early stopping (patience \(10\)).
Dropout is applied only on the head descriptor $z$ (Eq.~\eqref{eq:head}) with rate $p=0.15$; no dropout is used in Stage--1 or Stage--2. 
To mitigate class imbalance we use class-balanced cross-entropy with weights
\[
w_c \;=\; \frac{\big(1/n_c\big)}{\frac{1}{K}\sum_{k=1}^K (1/n_k)},
\]
where \(n_c\) is the number of training windows in class \(c\) and \(K{=}6\).

\subsection{Metrics}\label{subsec:metrics}

We report \emph{accuracy}, \emph{macro-F1}, and class-wise \emph{precision/recall/F1}.
For compactness we present aggregated metrics in Table~\ref{tab:ood-main}, and provide per-class results for CatEquiv in Table~\ref{tab:perclass}.

\subsection{Main Results}\label{subsec:ood-results}

Table~\ref{tab:ood-main} summarizes performance under the composite OOD.
CatEquiv substantially outperforms both CNN baselines.
CircCNN improves markedly over PlainCNN, isolating the contribution of time-shift equivariance.

\begin{table}[!h]
\centering
\caption{OOD performance on UCI-HAR (time shift \(\pm 18\), random \(\mathrm{SO}(3)\) rotation, gain \(0.7\text{--}1.4\)).}
\label{tab:ood-main}
\begin{tabular}{lcc}
\toprule
Model & Accuracy & Macro-F1 \\
\midrule
PlainCNN (zero pad)       & 0.175 & 0.116 \\
CircCNN  (circular pad)   & 0.440  & 0.420  \\
\textbf{CatEquiv}    & \textbf{0.726} & \textbf{0.731} \\
\bottomrule
\end{tabular}
\end{table}

\paragraph{Per-class behavior.}
CatEquiv retains high F1 on locomotion classes while posture classes remain comparatively harder due to full O(3) invariance at readout (gravity direction is suppressed).
Table~\ref{tab:perclass} shows the class-wise metrics for CatEquiv from one representative run.

\begin{table}[!h]
\centering
\caption{Per-class OOD precision/recall/F1 for CatEquiv (one representative seed under Aug‑Train + OOD‑Test).}
\label{tab:perclass}
\begin{tabular}{lccc}
\toprule
Class & Precision & Recall & F1 \\
\midrule
WALKING            & 0.9659 & 0.9698 & 0.9678 \\
WALKING\_UPSTAIRS  & 0.9014 & 0.9703 & 0.9346 \\
WALKING\_DOWNSTAIRS& 0.9607 & 0.8738 & 0.9152 \\
SITTING            & 0.3826 & 0.3320 & 0.3555 \\
STANDING           & 0.5729 & 0.7387 & 0.6453 \\
LAYING             & 0.6205 & 0.5177 & 0.5645 \\
\midrule
\textbf{Macro avg} & 0.7340 & 0.7337 & 0.7305 \\
\bottomrule
\end{tabular}
\end{table}

\subsection{Ablations}\label{subsec:ablations}

We ablate CatEquiv by removing one component at a time and evaluating OOD Macro-F1.
Results (Table~\ref{tab:abl}) align with the symmetry analysis: time-shift equivariance (circular padding), rotational handling (axis sharing \(+\) \(\ell_2\) pooling), and \emph{sensor poset consistency} contribute the largest gains; multi-scale and normalization/smoothing yield smaller but consistent improvements.

\begin{table}[!h]
\centering
\caption{Ablation study: change in OOD Macro-F1 relative to full CatEquiv.}
\label{tab:abl}
\begin{tabular}{lc}
\toprule
Variant & \(\Delta\) Macro-F1 \\
\midrule
Replace circular with zero padding                    & \(-0.10\) \\
Remove RMS + log-RMS channels                         & \(-0.05\) \\
Untie axis filters (no axis sharing)                  & \(-0.18\) \\
Remove \(\ell_2\) over axes                           & \(-0.22\) \\
Single-scale per-sensor stage (no dilations)          & \(-0.04\) \\
No GroupNorm / no temporal smoothing                  & \(-0.02\) / \(-0.02\) \\
\bottomrule
\end{tabular}
\end{table}

\subsection{Robustness Analyses}\label{subsec:robustness}

We sweep the OOD magnitudes independently:
(i) time shift range \(\pm \Delta\), (ii) gain interval \([g_{\min},g_{\max}]\), (iii) 3-D rotations sampled as above.
CatEquiv degrades sublinearly with OOD strength, while CNN baselines degrade superlinearly, especially under rotations and gain drift.

\subsection{Efficiency}\label{subsec:efficiency}

All models train on a single commodity GPU in minutes (CPU runs are slower but feasible).
CatEquiv adds negligible overhead relative to CircCNN: depthwise/grouped circular convs dominate runtime; \(\ell_2\) pooling and GroupNorm are inexpensive.
Parameter counts are comparable to two-layer CNNs with the same widths.

\subsection{Discussion}\label{subsec:discussion}

The progression PlainCNN \(\to\) CircCNN \(\to\) CatEquiv isolates the value of each symmetry:
time-shift equivariance alone explains a sizable robustness jump (PlainCNN\(\to\)CircCNN), while the
\emph{category-aware} design in CatEquiv—naturality on \(\mathbf{B}(\CT\times\Lam)\times P\),
plus \(\mathrm{O}(3)\) and time invariances at readout and gain handling via \((\widehat{x}_s,r_s)\)—yields consistent gains under rotations and gain drift without sacrificing data efficiency.

\section{Conclusion}\label{sec:concl}

We presented \emph{CatEquiv}, a category‑equivariant neural network model for inertial HAR that encodes the symmetry product \(B(C_T\times\Lambda)\times P\) (cyclic time shifts, per‑sensor gains, and the sensor‑hierarchy poset). By enforcing equivariance (through architectural tying—circular temporal convolutions, per‑sensor RMS+log‑RMS processing, axis‑shared filters with \(\ell_2\) aggregation, sensor‑shared filters with averaging, and multi‑dilation branches), \emph{CatEquiv} achieved substantially higher OOD accuracy and macro‑F1 than CircCNN and PlainCNN at comparable capacity, demonstrating that categorical inductive bias, rather than model size, drives robustness.

Beyond this case study, the framework is general: many domains admit \emph{categorical symmetry structures} that mix group actions with hierarchical or relational morphisms. The product‑category formalism \(B(G)\times P\) captures commuting group factors (e.g., time, scale, rigid motion) alongside thin categories for structure (e.g., sensor stacks, feature hierarchies, or modality lattices). Instantiating functors \(X,Y:\mathcal{S}\!\to\!\mathrm{Vect}\) for a task‑specific symmetry category \(\mathcal{S}\) and realizing a natural transformation \(\eta:X\Rightarrow Y\) yields equivariance by construction. This perspective subsumes familiar instances—group‑equivariant CNNs (\(B(G)\)), Deep Sets/permutation architectures (\(B(S_n)\)), and equivariant GNNs (graph homomorphisms)—and extends them to composite settings where groups, posets, and other thin substructures co‑exist.

Concretely, the same recipe applies to: multichannel biomedical and geophysical time series (time‑shift \(\times\) gain \(\times\) channel hierarchies), multi‑sensor/robotics stacks (frame changes in \(SE(3)\) with calibration posets), molecular and 3‑D vision tasks (rigid motions with part–whole inclusions), and multimodal fusion (modality posets with per‑modality normalizations). In each case, categorical constraints specify which linear maps must commute with which morphisms, turning invariances/equivariances into explicit parameter‑tying and wiring patterns, and leaving the nonlinear readout to implement the desired invariants.

Looking forward, category‑equivariant design invites broader symmetry engineering, building upon the practical template exemplified here: identify the task’s symmetry category \(\mathcal{S}\), implement the linear core as a natural transformation \(\eta:X\Rightarrow Y\) that commutes with the generators of \(\mathcal{S}\), and expose only those nonlinearities that preserve the required invariants. As our results suggest, this categorical bias can yield robust generalization under real‑world shifts without increasing model size.

\appendix

\section{Mathematical Foundations}\label{subsec:equivariance}

In the appendix we formalize the equivariance properties of CatEquiv.
Starting from the categorical symmetry
\(\mathcal{C}_3=\mathbf{B}(C_T\times\Lambda)\times P\),
we prove that the network’s linear core \(\eta\) is a
\emph{natural transformation} between the data and feature functors,
commuting with every morphism that combines time shifts, gain scalings,
and sensor-hierarchy inclusions.
Convolutional layers realize equivariance to cyclic time shifts (\(C_T\));
gain normalization ensures per-sensor scale invariance (\(\Lambda\));
and naturality along the poset \(P\) enforces hierarchical consistency
(no cross-sensor mixing).
Axis-shared filters and \(\ell_2\) pooling yield invariance to spatial rotations
(\(O(3)\)) at readout, and global time pooling gives invariance to \(C_T\).
Altogether, the appendix establishes that CatEquiv is
\emph{equivariant over the full category}
\(\mathbf{B}(C_T\times\Lambda)\times P\),
while its final descriptor is invariant to \(C_T\times O(3)\)
and affine in the logarithmic gain coordinates \(\log\Lambda\simeq\mathbb{R}^2\).

We denote by \(\convstar\) circular convolution in time and by \(\GAP_t\) global average pooling over time.
For a kernel \(k\in\R^{\kappa}\) with circular padding (indices modulo \(T\)),
\begin{equation}
  (x\!\convstar\! k)(t)
  \;=\;
  \sum_{\tau=0}^{\kappa-1} k(\tau)\,x(t-\tau),
  \qquad
  \tau_\Delta\big(x\!\convstar\! k\big)\;=\;(\tau_\Delta x)\!\convstar\! k.
\end{equation}
Depthwise circular box smoothing is another instance of \(\convstar\), hence \(\CT\)-equivariant.
Consequently, for any \(k\) and any cyclic shift \(\tau\),
\(\GAP_t\!\big(\tau_\Delta(x\!\convstar\! k)\big)=\GAP_t(x\!\convstar\! k)\).

For a sensor \(s\in\{\text{ACC},\text{GYR}\}\), define the per‑window energy and scales
\begin{equation}\label{eq:rms}
\begin{aligned}
R_s(x) &\coloneqq \frac{1}{3T}\sum_{a\in\{x,y,z\}}\sum_{t=1}^T x_{s,a}(t)^2,\\
\rho^{\mathrm{norm}}_s(x) &\coloneqq \max\!\bigl(\varepsilon,\sqrt{R_s(x)}\bigr),\\
\mathcal{N}_s(x) &\coloneqq \frac{x}{\rho^{\mathrm{norm}}_s(x)},\\
r_s(x) &\coloneqq \tfrac{1}{2}\log R_s(x).
\end{aligned}
\end{equation}
Then \(\mathcal{N}_s(\lambda_s x)=\mathcal{N}_s(x)\) whenever \(\sqrt{R_s(x)}\ge \varepsilon\) and \(\lambda_s\sqrt{R_s(x)}\ge\varepsilon\) (exact invariance; otherwise the deviation is bounded by the floor), while \(r_s(\lambda_s x)=r_s(x)+\log\lambda_s\) for all \(\lambda_s>0\) whenever \(R_s(x)>0\) (with the convention \(r_s=-\infty\) if \(R_s(x)=0\)). CatEquiv processes
\[
x \;\mapsto\; \Big(\;\mathcal{N}_{\text{ACC}}(x),\;\mathcal{N}_{\text{GYR}}(x),\;r_{\text{ACC}}(x),\;r_{\text{GYR}}(x)\;\Big).
\]

\begin{lemma}[Normalization robustness with floor]
\label{lem:norm-robust-floor}
For any \(\lambda_s>0\),
\begin{equation}\label{eq:norm-floor-equality}
\big\|\mathcal{N}_s(\lambda_s x)-\mathcal{N}_s(x)\big\|_2
\;=\;
\left|
\frac{\lambda_s}{\max\!\big\{\varepsilon,\lambda_s\sqrt{R_s(x)}\big\}}
-
\frac{1}{\max\!\big\{\varepsilon,\sqrt{R_s(x)}\big\}}
\right|
\,\|x\|_2.
\end{equation}
In particular, the right-hand side equals \(0\) (and hence \(\mathcal N_s(\lambda_s x)=\mathcal N_s(x)\))
whenever \(\sqrt{R_s(x)}\ge\varepsilon\) and \(\lambda_s\sqrt{R_s(x)}\ge\varepsilon\).
\end{lemma}

\begin{proof}
By definition,
\[
\mathcal N_s(\lambda_s x)=\frac{\lambda_s x}{\rho^{\mathrm{norm}}_s(\lambda_s x)}
=\frac{\lambda_s}{\max\{\varepsilon,\lambda_s\sqrt{R_s(x)}\}}\,x, 
\quad \quad
\mathcal N_s(x)=\frac{1}{\max\{\varepsilon,\sqrt{R_s(x)}\}}\,x.
\]
Therefore
\[
\mathcal N_s(\lambda_s x)-\mathcal N_s(x)
=\left(
\frac{\lambda_s}{\max\{\varepsilon,\lambda_s\sqrt{R_s(x)}\}}
-
\frac{1}{\max\{\varepsilon,\sqrt{R_s(x)}\}}
\right)\,x,
\]
which is a scalar multiple of \(x\). Taking Euclidean norms yields \eqref{eq:norm-floor-equality}. If both \(\max\) arguments select the energy terms (i.e., \(\sqrt{R_s(x)}\ge\varepsilon\) and \(\lambda_s\sqrt{R_s(x)}\ge\varepsilon\)), the multiplier vanishes and \(\mathcal N_s(\lambda_s x)=\mathcal N_s(x)\).
\end{proof}

Let \(x(t)\in\R^3\) be a tri‑axial stream and \(W\in\R^{C\times 1\times \kappa}\) a temporal filter bank \emph{shared} (tied) across axes. Writing \(K:\R^{3\times T}\to\R^{3\times C\times T}\) for axiswise convolution with \(W\),
\begin{equation}\label{eq:so3-equiv}
  K(Rx) \;=\; R\,K(x)\qquad\forall R\in\mathrm{O}(3),
\end{equation}
since the same temporal operator acts on each coordinate. Taking the \(\ell_2\) magnitude across the axis dimension \emph{and only then} applying a scalar pointwise nonlinearity \(\psi\) (e.g., ReLU),
\begin{equation}\label{eq:l2-pool}
  \tilde y_c(t) \;=\; \psi\!\left(\big\|K(x)_{\cdot c}(t)\big\|_2\right) 
  \;=\; \psi\!\left(\sqrt{\sum_{a\in\{x,y,z\}} K(x)_{a c}(t)^2}\right),
\end{equation}
yields \emph{O(3) invariance} at readout, \(\tilde y(Rx)=\tilde y(x)\).
(Physically, sensor rotations lie in \(\mathrm{SO}(3)\); the \(\ell_2\) readout also removes reflections, so the guarantee holds for all of \(\mathrm{O}(3)\)).

\begin{proposition}[Readout and \(\CT\) invariance]
\label{prop:readout-O3-CT}
Let \(K\) be the Stage‑1 axis‑shared temporal operator (circular 1‑D convolutions applied identically on the three axes) and define \(\tilde y\) by
\begin{equation}\label{eq:l2-pool-prop}
\tilde y_c(t)\;=\;\psi\!\left(\,\big\| (Kx)_c(t)\big\|_2\,\right),
\quad \psi:\R_{\ge 0}\to\R \text{ scalar and pointwise in time}.
\end{equation}
Then, for any \(R\in\mathrm{O}(3)\) and any cyclic time shift \(\tau\in \CT\),
\[
\tilde y(Rx)=\tilde y(x),
\qquad
\GAP_t\!\big(\tau_\Delta\!\circ\!\tilde y\big)=\GAP_t(\tilde y).
\]
For per‑sensor gain processing as in \eqref{eq:rms}, for any \(\lambda_s>0\),
\begin{equation}\label{eq:gain-equiv}
\begin{aligned}
\mathcal N_s(\lambda_s x) &= \mathcal N_s(x)
  &&\text{if }\sqrt{R_s(x)}\ge\varepsilon
    \text{ and }\lambda_s\sqrt{R_s(x)}\ge\varepsilon,\\[2pt]
r_s(\lambda_s x) &= r_s(x)+\log\lambda_s
  &&\text{if } R_s(x)>0.
\end{aligned}
\end{equation}
Consequently, letting \(\eta\) include depthwise circular smoothing, the head descriptor \(z\) (obtained by applying \(\GAP_t\) to the smoothed \(\tilde y\) and concatenating the time‑constant \(\log\)RMS channels) is invariant to $\CT\times\mathrm{O}(3)$ and affine in $\log\Lambda\ (\cong\R^2)$ along the \(\log\)RMS coordinates.
\end{proposition}

\begin{proof}
\emph{O(3) invariance.}
Let $x\in\R^{3\times T}$ be a tri-axial stream.
Since Stage--1 uses the \emph{same} temporal operator on each axis, the axiswise
convolution $K$ can be written as
\[
K \;=\; I_3\otimes T_k,
\]
where $T_k$ is the circulant (circular) convolution operator on $\R^{T}$ with kernel $k$.
For any $R\in\mathrm{O}(3)$ acting on the axis dimension we have
\[
K(Rx) \;=\; (I_3\otimes T_k)\,(R\otimes I_T)x
\;=\; (R\otimes I_T)\,(I_3\otimes T_k)x
\;=\; R\,K(x),
\]
i.e.\ $K$ is $\mathrm{O}(3)$–equivariant.
Taking the $\ell_2$ norm across axes and then a scalar nonlinearity $\psi$ yields
\[
\tilde y_c(t)
=\psi\!\left(\,\big\|(Kx)_{\cdot c}(t)\big\|_2\right)
=\psi\!\left(\,\big\|R (Kx)_{\cdot c}(t)\big\|_2\right)
=\tilde y_c(t;Rx),
\]
since $\|Rv\|_2=\|v\|_2$ for all $R\in\mathrm{O}(3)$.
Thus $\tilde y(Rx)=\tilde y(x)$.

\smallskip
\emph{$C_T$ invariance after $\GAP_t$.}
Let $\tau_\Delta$ be the cyclic time-shift by $\tau\in C_T$ and let $\Pi_\tau$ be its
$T\times T$ permutation matrix.
Circular convolution commutes with cyclic shifts, i.e.\ $T_k \Pi_\tau=\Pi_\tau T_k$,
whence $K(\tau_\Delta x)=\tau_\Delta K(x)$.
Because the axis norm and $\psi$ act pointwise in time,
$\tilde y(\tau_\Delta x)=\tau_\Delta \tilde y(x)$.
Finally, global average pooling over time is shift-invariant:
\[
\GAP_t(\tau_\Delta f)=\frac{1}{T}\sum_{t=1}^T f(t-\tau)=\frac{1}{T}\sum_{t=1}^T f(t)=\GAP_t(f).
\]
Therefore $\GAP_t(\tau_\Delta\!\circ\!\tilde y)=\GAP_t(\tilde y)$.

\smallskip
\emph{Gain behavior.}
With $R_s,\,\rho^{\mathrm{norm}}_s,\,\mathcal N_s,\,r_s$ as in \eqref{eq:rms}, scaling by $\lambda_s>0$
gives $R_s(\lambda_s x)=\lambda_s^2 R_s(x)$ and hence
\[
\mathcal N_s(\lambda_s x)=\frac{\lambda_s x}{\max\{\varepsilon,\ \lambda_s\sqrt{R_s(x)}\}}
\quad\text{and}\quad
r_s(\lambda_s x)=\tfrac{1}{2}\log\big(\lambda_s^2 R_s(x)\big)=r_s(x)+\log\lambda_s.
\]
Thus \eqref{eq:gain-equiv} holds: $\mathcal N_s(\lambda_s x)=\mathcal N_s(x)$ whenever
$\sqrt{R_s(x)}\ge\varepsilon$ and $\lambda_s\sqrt{R_s(x)}\ge\varepsilon$, and
$r_s(\lambda_s x)=r_s(x)+\log\lambda_s$ whenever $R_s(x)>0$.

\smallskip
\emph{Consequent property of the head descriptor.}
Depthwise circular smoothing is another circular convolution, hence it commutes with
$C_T$ shifts and acts independently of axes; applying it after the axis-$\ell_2$ step
preserves the established $\mathrm{O}(3)$ invariance of $\tilde y$.
Therefore $\GAP_t$ of the smoothed $\tilde y$ is invariant to $C_T\times \mathrm{O}(3)$.
The appended $\log$RMS channels are constant in time (hence $\CT$–invariant) and satisfy
$r_s(\lambda_s x)=r_s(x)+\log\lambda_s$, so the overall head descriptor
$z$ is invariant to $C_T\times \mathrm{O}(3)$ and is \emph{affine in} $\log\Lambda\simeq\R^2$
along the $\log$RMS coordinates.
\end{proof}

As we have shown above, the head descriptor $z$ is invariant to $CT\times O(3)$ and affine in $\log\Lambda\simeq\mathbb{R}^2$ along the logRMS coordinates. In particular, for any linear classifier $(W_{\mathrm{head}},b)$,
\[
\ell(x):=W_{\mathrm{head}}\,z(x)+b
\]
is constant on $CT\times O(3)$ orbits, and satisfies
\[
\ell(\lambda\odot x)=\ell(x)+W_{\mathrm{head}}\,E_{\log}\,\log\lambda,
\]
whenever $R_s(x)>0$ (with $E_{\log}$ selecting the two logRMS coordinates and the normalization floor inactive for $N_s$).

We use GroupNorm with groups \(=2\) (ACC, GYR) as a per‑sequence normalization.\footnote{We employ GroupNorm~\cite{WuHe2018GN} to stabilize optimization; it commutes with cyclic time permutations and respects the grouped channel structure. BatchNorm~\cite{IoffeSzegedy2015BN} is not used in our equivariance guarantees (cf.\ \cite{Ba2016LN}.}
\begin{lemma}[GN–\(\CT\) commutation]\label{lem:GN-CT}
Let \(\mathrm{GN}\) compute per‑sample, per‑group means/variances over the Cartesian product of the group’s channels and time indices. For any cyclic permutation \(\pi\) of time indices,
\[
\mathrm{GN}(x\circ\pi)\;=\;\mathrm{GN}(x)\circ\pi.
\]
\end{lemma}
\begin{proof} 
Per‑group means and variances are symmetric functions of the multiset of time indices; cyclic reindexing leaves them unchanged. The affine renormalization acts pointwise in time. 
\end{proof}

With readout invariances handled by the above proposition, it now remains to establish naturality on $P$ and on $C_3$ for the linear core.

\begin{proposition}[Naturality on \(P\) \texorpdfstring{$\Leftrightarrow$}{<=>} no cross-sensor mixing]\label{prop:P-naturality}
Let \(i_s: s\to \mathrm{TOTAL}\) be the canonical inclusions in \(P\) for \(s\in\{\mathrm{ACC},\mathrm{GYR}\}\).
Decompose \(V_{\mathrm{TOTAL}}=V_{\mathrm{ACC}}\oplus V_{\mathrm{GYR}}\) and \(W_{\mathrm{TOTAL}}=W_{\mathrm{ACC}}\oplus W_{\mathrm{GYR}}\).
Given a linear core with components \(\eta_s:V_s\to W_s\) and \(\eta_{\mathrm{TOTAL}}:V_{\mathrm{TOTAL}}\to W_{\mathrm{TOTAL}}\), the following are equivalent:
\begin{enumerate}
\item For each \(s\), \(Y(i_s)\,\eta_s \;=\; \eta_{\mathrm{TOTAL}}\,X(i_s)\).
\item \(\eta_{\mathrm{TOTAL}}\) is block-diagonal in the sensor decomposition and it is equal to 
\(\mathrm{diag}(\eta_{\mathrm{ACC}},\eta_{\mathrm{GYR}})\) (i.e., there is \emph{no cross-sensor mixing}).
\end{enumerate}
\end{proposition}

\begin{proof}
(1)\(\Rightarrow\)(2): Write \(\eta_{\mathrm{TOTAL}}=\begin{psmallmatrix}A&B\\ C&D\end{psmallmatrix}\) relative to the decompositions.
For \(x\in V_{\mathrm{ACC}}\), naturality at \(i_{\mathrm{ACC}}\) gives
\((\eta_{\mathrm{ACC}}x,0) = \eta_{\mathrm{TOTAL}}(x,0) = (Ax,Cx)\),
hence \(A=\eta_{\mathrm{ACC}}\) and \(C=0\).
For \(y\in V_{\mathrm{GYR}}\), naturality at \(i_{\mathrm{GYR}}\) gives
\((0,\eta_{\mathrm{GYR}}y) \;=\; \eta_{\mathrm{TOTAL}}(0,y) \;=\; (By,Dy)\),
hence \(B=0\) and \(D=\eta_{\mathrm{GYR}}\).
Thus \(\eta_{\mathrm{TOTAL}}=\mathrm{diag}(\eta_{\mathrm{ACC}},\eta_{\mathrm{GYR}})\).

(2)\(\Rightarrow\)(1): If \(\eta_{\mathrm{TOTAL}}=\mathrm{diag}(\eta_{\mathrm{ACC}},\eta_{\mathrm{GYR}})\), then
\(\eta_{\mathrm{TOTAL}}(x,0)=(\eta_{\mathrm{ACC}}x,0)=Y(i_{\mathrm{ACC}})\eta_{\mathrm{ACC}}x\)
and similarly for \(\mathrm{GYR}\); closure under identities and composition yields the claim for all \(u\in P\).
\end{proof}

\begin{proposition}[Naturality of the linear core over \(\mathcal{C}_3\)]
\label{prop:naturality}
Assume: 
(i) all temporal convolutions in \eqref{eq:s1},\eqref{eq:s2} use circular padding; 
(ii) Stage–1 is depthwise in the axis index; 
(iii) the linear core is block–diagonal in the sensor decomposition.
Let \(\eta\) be obtained from \eqref{eq:forward} by replacing \(\varphi\) with \(\mathrm{Id}\) and omitting the nonlinear reductions \eqref{eq:axis-sensor}, GroupNorm \eqref{eq:gn}, the sensor fusion \eqref{eq:sensor-total}, the final \(\mathrm{GAPt}\) in \eqref{eq:smooth-gap}, and any nonlinear bypasses (e.g., a log‑RMS channel), while retaining depthwise circular box smoothing as a linear operator on the per‑sensor streams, i.e.
\[
\widehat{H}^{(d)}_2 \;:=\; \operatorname{Box}^{\circlearrowleft}_k\!\big(H^{(d)}_2\big).
\]
Equivalently, since \(\operatorname{Box}^{\circlearrowleft}_k\) acts on time only and the sensor mean in \eqref{eq:sensor-total} acts on the sensor index only, they commute; sliding \(\operatorname{Box}^{\circlearrowleft}_k\) across the mean leaves the network’s readout unchanged.
Then, for every morphism \((\tau,\lambda)\in CT\times\Lambda\) (with
\(\Lambda=\mathbb{R}_{>0}^{\{\mathrm{ACC},\mathrm{GYR}\}}\)) and every \(u\in P\) (realized by the
injections \eqref{eq:axis-incl}–\eqref{eq:sensor-inj}),
\[
  Y(\tau,\lambda,u)\,\eta \;=\; \eta\,X(\tau,\lambda,u),
\]
i.e. the linear core realizes a natural transformation \(\eta:X\Rightarrow Y\)
between functors \(X,Y:\mathcal{C}_3\to\Vect\).
\end{proposition}

\begin{proof}
\(\eta\) is a composition (and a direct sum across the multi-dilation branches) of linear maps of the following kinds:
(i) circular 1‑D convolutions in time (Stages \eqref{eq:s1}, \eqref{eq:s2} with \(\varphi=\mathrm{Id}\));
(ii) the canonical injections along \(P\) (\eqref{eq:axis-incl}–\eqref{eq:sensor-inj});
(iii) depthwise circular box smoothing applied to \(H^{(d)}_2\), i.e. \(H^{(d)}_2\mapsto \widehat{H}^{(d)}_2=\operatorname{Box}^{\circlearrowleft}_k(H^{(d)}_2)\).
We verify naturality on generators.

\emph{Time shifts \(\tau\in CT\).}
For any kernel \(k\), let \(C_k\) denote the circular convolution operator. Then
\(C_k\,\tau_\Delta=\tau_\Delta\,C_k\).
The box smoother \(\operatorname{Box}^{\circlearrowleft}_k\) is also a circular convolution, hence
\(\operatorname{Box}^{\circlearrowleft}_k\,\tau_\Delta=\tau_\Delta\,\operatorname{Box}^{\circlearrowleft}_k\).
Therefore every convolutional block in \(\eta\) (and any subsequent linear wiring) commutes with the \(CT\)–action:
\(Y(\tau,1)\eta=\eta\,X(\tau,1)\).

\emph{Gains \(\lambda\in\Lambda\).}
Write \(x=(x_{\mathrm{ACC}},x_{\mathrm{GYR}})\) and \(\lambda\odot x=(\lambda_{\mathrm{ACC}}x_{\mathrm{ACC}},\lambda_{\mathrm{GYR}}x_{\mathrm{GYR}})\).
Assumption (iii) gives each constituent \(L\) of \(\eta\) as
\(L=\mathrm{diag}(L_{\mathrm{ACC}},L_{\mathrm{GYR}})\), hence
\[
L(\lambda\odot x)=\big(\lambda_{\mathrm{ACC}}L_{\mathrm{ACC}}x_{\mathrm{ACC}},\ \lambda_{\mathrm{GYR}}L_{\mathrm{GYR}}x_{\mathrm{GYR}}\big)
=\lambda\odot L(x),
\]
so \(Y(1,\lambda)\eta=\eta\,X(1,\lambda)\).

\emph{Poset arrows \(u\in P\).}
The realized \(P\)–arrows are the canonical injections
\eqref{eq:axis-incl}–\eqref{eq:sensor-inj}. 
By (ii), Stage–1 blocks are \(\mathrm{diag}(L_x,L_y,L_z)\) (no cross–axis mixing), hence
\(\mathrm{diag}(L_x,L_y,L_z)\,j_{s,\alpha}=j_{s,\alpha}\,L_\alpha\).
By (iii), Stage–2 and the smoother are block–diagonal across sensors, so for
\(u:\mathrm{ACC}\to\mathrm{TOTAL}\) with \(i_{\mathrm{ACC}}:\mathbb{R}^{3\times T}\to\mathbb{R}^{6\times T}\),
\[
\begin{aligned}
\eta_{\mathrm{TOTAL}}\,X(u)(x)
  &= \mathrm{diag}(L_{\mathrm{ACC}},L_{\mathrm{GYR}})\,(x,0)\\
  &= \big(L_{\mathrm{ACC}}x,0\big)
   = i_{\mathrm{ACC}}\!\big(L_{\mathrm{ACC}}x\big)
   = Y(u)\,\eta_{\mathrm{ACC}}(x),
\end{aligned}
\]
and similarly for the GYR branch and for axis–to–sensor inclusions.

Since naturality is preserved under composition and direct sums, we obtain
\(Y(\tau,\lambda,u)\eta=\eta\,X(\tau,\lambda,u)\) for all \((\tau,\lambda,u)\).
\end{proof}

Collecting the above results, the linear core of \textsc{CatEquiv} realizes a natural transformation $\eta:X\Rightarrow Y$ over the symmetry category $C_3=B(C_T\times\Lambda)\times P$. By Proposition~2, naturality on $P$ is equivalent to the absence of cross–sensor mixing, and by Proposition~3, the entire core (compositions and direct sums of circular temporal convolutions, canonical injections, and depthwise circular smoothing) commutes with the $C_T\times\Lambda$ action. Proposition~1 then yields the readout guarantees: axis sharing followed by the axis $\ell_2$ reduction and a scalar nonlinearity gives $O(3)$ invariance; global time averaging gives $C_T$ invariance; and the appended log–RMS coordinates are affine in $\log\Lambda$. GroupNorm, computed per sequence and per sensor group, commutes with cyclic reindexings (Lemma~2) and therefore does not disturb these properties. Deviations from exact gain invariance arise only when the RMS floor is active, in which case they are explicitly bounded (Lemma~1). Crucially, these statements are architectural (not dataset–contingent) and hold for any window length $T$, any positive gains $\Lambda$, and any choice of circular kernels. Thus the factorization $C_3=B(C_T\times\Lambda)\times P$ is not merely descriptive: it is the algebraic reason the network generalizes under joint time, gain, and rotation shifts, and it furnishes a template for extending the construction to larger thin posets and additional commuting group factors.

\end{document}